%% file: ms.tex
\documentclass{article}

\usepackage{algorithm}
\usepackage[noend]{algpseudocode}
\usepackage[pdftex]{graphicx}
\usepackage[final]{neurips_wrl2023}
\input{math_commands.tex}

\usepackage[utf8]{inputenc} %
\usepackage[T1]{fontenc}    %
\usepackage{hyperref}       %
\usepackage{url}            %
\usepackage{booktabs}       %
\usepackage{amsfonts}       %
\usepackage{nicefrac}       %
\usepackage{microtype}      %

\usepackage{xcolor}

\usepackage[inline]{enumitem}
\setlist*[enumerate,1]{label=(\roman*),font=\itshape}

\title{A Statistical Guarantee for Representation Transfer in Multitask Imitation Learning}

\author{Bryan Chan\thanks{ Work done during internship at Ocado Technology.} \\
Department of Computing Science\\
University of Alberta\\
Edmonton, Canada \\
\texttt{bryan.chan@ualberta.ca}
\And
Karime Pereida\thanks{ Equal advising.} \& James Bergstra\footnotemark[2]\\
AI Platform\\
Ocado Technology\\
Toronto, Canada\\
\texttt{\{k.pereidaperez,james.bergstra\}@ocado.com}
}

\begin{document}

\maketitle

\begin{abstract}
  Transferring representation for multitask imitation learning has the potential to provide improved sample efficiency on learning new tasks, when compared to learning from scratch.
  In this work, we provide a statistical guarantee indicating that we can indeed achieve improved sample efficiency on the target task when a representation is trained using sufficiently diverse source tasks.
  Our theoretical results can be readily extended to account for commonly used neural network architectures with realistic assumptions.
  We conduct empirical analyses that align with our theoretical findings on four simulated environments---in particular leveraging more data from source tasks can improve sample efficiency on learning in the new task.
\end{abstract}

\input{TexFiles/Introduction.tex}
\input{TexFiles/Background}
\input{TexFiles/RepresentationTransfer}
\input{TexFiles/EmpiricalAnalysis}
\input{TexFiles/Conclusion}

\bibliographystyle{plainnat}
\bibliography{reference}

\appendix

\input{TexFiles/RelatedWork}
\input{TexFiles/TheoreticalAnalysis}
\input{TexFiles/DetailedProofs}
\input{TexFiles/AlgorithmDetails}
\input{TexFiles/ImplementationDetails}

\end{document}

%% file: math_commands.tex
\usepackage{amsmath,amsfonts,bm}

\newcommand{\newterm}[1]{{\bf #1}}

\newcommand\twonorm[1]{\lVert#1\rVert_2}
\newcommand\bigtwonorm[1]{\Bigg\lVert#1\Bigg\rVert_2}
\newcommand\frobnorm[1]{\lVert#1\rVert_F}
\newcommand\maxnorm[1]{\lVert#1\rVert_\infty}
\newcommand\numeq[1]%
  {\stackrel{\scriptscriptstyle(\mkern-1.5mu#1\mkern-1.5mu)}{=}}
\newcommand\numleq[1]%
  {\stackrel{\scriptscriptstyle(\mkern-1.5mu#1\mkern-1.5mu)}{\leq}}
\newcommand\numsim[1]%
  {\stackrel{\scriptscriptstyle#1}{\sim}}

\def\fR{\mathfrak{R}}

\newenvironment{proof}{\paragraph{Proof:}}{\hfill$\square$}
\newtheorem{theorem}{Theorem}
\newtheorem{lemma}{Lemma}
\newtheorem{prop}{Proposition}
\newtheorem{assumption}{Assumption}
\newtheorem{definition}{Definition}
\newtheorem{corollary}{Corollary}

\def\eqref#1{equation~\ref{#1}}

\def\1{\bm{1}}

\def\rmX{{\mathbf{X}}}

\DeclareMathAlphabet{\mathsfit}{\encodingdefault}{\sfdefault}{m}{sl}
\SetMathAlphabet{\mathsfit}{bold}{\encodingdefault}{\sfdefault}{bx}{n}

\def\gA{{\mathcal{A}}}

\def\gD{{\mathcal{D}}}

\def\gF{{\mathcal{F}}}

\def\gL{{\mathcal{L}}}
\def\gM{{\mathcal{M}}}

\def\gO{{\mathcal{O}}}

\def\gS{{\mathcal{S}}}

\def\gX{{\mathcal{X}}}

\def\sP{{\mathbb{P}}}

\def\sR{{\mathbb{R}}}

\newcommand{\E}{\mathbb{E}}

\newcommand{\softmax}{\mathrm{softmax}}

\DeclareMathOperator*{\argmax}{arg\,max}
\DeclareMathOperator*{\argmin}{arg\,min}

%% file: TexFiles/Introduction.tex
\section{Introduction}

Imitation learning (IL) is a common approach to learn sequential decision making agents---it involves imitating the expert through matching distributions induced by the expert demonstrations \citep{osa2018algorithmic}.
However, current methods require thousands of demonstrations even in simple tasks \citep{mandlekar2022matters,DBLP:conf/corl/JangIKKELLF21,DBLP:journals/ral/AblettCK23}.
Acquiring large amounts of data can be expensive and even infeasible in domains including robotic and healthcare applications.
To address this challenge, empirical research has proposed transferring part of the agent trained from one or more tasks to a target task with the goal of improving sample efficiency on the target task \citep{DBLP:conf/rss/BrohanBCCDFGHHH23,hansen2022pre,DBLP:conf/corl/JangIKKELLF21,NEURIPS2022_ca3b1f24}.
In this paper we show how much sample efficiency on the target task improves, compared to training an agent from scratch, when we transfer a pretrained representation to the target task via multitask imitation learning (MTIL).

We have three main contributions: (1) Although \citet{DBLP:conf/icml/AroraDKLS20} has investigated the benefits of representation transfer, their result does not relate the target task and the source tasks.
We cannot guarantee that transferring the representation to a particular target task can yield any benefit.
Inspired by \citet{NEURIPS2020_59587bff}, we provide an analysis that relates the source and target tasks via the notion of task diversity.
(2) This line of work \citep{DBLP:conf/icml/AroraDKLS20,NEURIPS2020_59587bff,maurer2016benefit} relies on Gaussian complexity,
in this paper we instead use Rademacher complexity and provide a tighter bound by a log factor than using Gaussian complexity.
Our result is due to the objective of behavioural cloning, where the method aims to minimize the Kullback–Leibler (KL) divergence between the expert and the learner \citep{NEURIPS2020_b5c01503}.
The consequence is that we can connect our result with deep-learning theory, where commonly used neural networks are directly quantified with Rademacher complexity \citep{DBLP:journals/actanum/BartlettMR21}.
(3) Based on our theory, we further conduct experiments to demonstrate that transferring representations from source tasks to target tasks is a valid approach---the agent performs better as we increase the number of tasks and data.

%% file: TexFiles/Background.tex
\section{Preliminaries}
Sequential decision making problems can be formulated as Markov Decision Processes (MDPs).
An infinite-horizon MDP is a tuple $\gM = \langle \gS, \gA, r, P, \rho, \gamma \rangle$, where $\gS$ is a finite state space, $\gA$ is a finite action space, $r: \gS \times \gA \to [0, 1]$ is the bounded reward function, $P \in \Delta_{\gS \times \gA}^{\gS}$ is the transition distribution over the states for each state-action pair, $\rho \in \Delta^{\gS}$ is the initial state distribution over the states, and $\gamma \in [0, 1)$ is the discount factor.
The agent interacts with the environment through a policy $\pi \in \Pi$, which we assume to be stationary and Markovian (i.e. $\Pi = \Delta^{\gA}_\gS$.)

The interconnection between the policy $\pi$ and the environment $\gM$ induces a random infinite-length trajectory $S_0, A_0, S_1, A_1, \dots$, where $S_0 \sim \rho$, $A_h \sim \pi(\cdot \vert S_h)$, and $S_{h + 1} \sim P(\cdot \vert S_h, A_h)$.
The corresponding discounted stationary state(-action) distribution for policy $\pi$, which describes the ``frequency'' of visiting state $s$ (and action $a$) under $\pi$, can be written as $\nu_\pi(s) = (1 - \gamma) \sum_{h=0}^\infty \gamma^h \sP(S_h = s; \pi)$ (and $\mu_\pi(s, a) = (1 - \gamma) \sum_{h = 0}^\infty \gamma^h \sP (S_h = s, A_h = a ; \pi)$.)

For any policy $\pi \in \Pi$, the corresponding value function $v^\pi : \gS \to \sR$ is defined as $v^\pi (s) = \E_{\pi, \rho, P} \left[ \sum_{h = 0}^{\infty} \gamma^h r(S_h, A_h) \vert S_0 = s \right]$.
The optimal policy $\pi^*$ is such that $\pi^*(s) = \argmax_{\pi \in \Pi} v^\pi(s)$, for any $s \in \gS$ (randomly breaking ties in the case of two or more maxima.)
In general, the goal is to obtain an $\varepsilon$-optimal policy (i.e. $v^\pi(s) \geq v^{\pi^*}(s) - \varepsilon$, for any $s \in \gS$.)

In imitation learning, the learner has no access to the reward function $r$ and is given demonstrations from the optimal (expert) policy $\pi^*$ instead.
The demonstrations form a set of $N$ state-action pairs $\{(s_n, a_n)\}_{n=1}^N$, where $(s_n, a_n) \numsim{\text{i.i.d.}} \mu_{\pi^*}$.
The goal is to obtain an $\varepsilon$-optimal policy using the demonstrations.
Behavioural cloning treats this problem as a supervised learning problem and aims to minimize the risk \citep{NIPS1988_812b4ba2}.
The risk is defined as $\ell(\pi) = \E_{(s, a) \sim \mu_{\pi^*}} \left[ \ell(\pi(s), a) \right]$, where $\pi \in \Pi$ and $\ell (\cdot, \cdot): \Delta^\gA \times \gA \to \sR^+$ is a loss function.
This paper considers the log loss $\ell(\pi(s), a) = -\log \pi(a \vert s)$, which is a surrogate of the 0-1 loss.
The log loss is also equivalent to the KL-divergence between the expert and the learner when the expert is deterministic.

%% file: TexFiles/RepresentationTransfer.tex
\section{Multitask Imitation Learning (MTIL) with Representation Transfer}
We consider the transfer-learning setting where we are given demonstrations of $T$ source tasks.
We define each task $t$ as an MDP with different transition distributions and reward functions, but with the same state and action spaces.
When the context is unclear, we use subscript, e.g. $\pi_t$, to denote the task-specific objects.
Our goal is to leverage the demonstrations of $T$ source tasks to learn an $\varepsilon$-optimal policy on a new target task $\tau$ with better sample efficiency, when compared to learning from scratch.
To achieve this, we first learn a shared representation from the source tasks and transfer it to the target task.
During the transfer, we fix the representation and only learn the task-specific mapping, similar to a finetuning procedure \citep{DBLP:conf/acl/RuderH18,NEURIPS2022_0cde695b}.

Formally, we consider softmax parameterized policies of the form $\pi^{f, \phi}(s) = \softmax((f \circ \phi) (s))$, where $f \in \gF$ is the task-specific mapping and $\phi \in \Phi$ is the representation.
We perform a two-phase procedure for transfer learning:
(1) learn a representation $\hat{\phi}$ from the source tasks, and
(2) learn a task-specific mapping $\hat{f}$ such that $\pi^{\hat{f}, \hat{\phi}}$ performs well on the target task.
We call the phases respectively the training phase and the transfer phase.

In the training phase, for each task $t$ of the $T$ source tasks, we are given $N$ state-action pairs $\{(s_{t, n}, a_{t, n})\}_{n=1}^N$.
Let $\boldsymbol{f} = (f_1, \dots, f_T)$, where $f_t \in \gF$ is the task-specific mapping for task $t$, for $t \in [T]$, which we write $\boldsymbol{f} \in \gF^{\otimes T}$ for conciseness.
Then, we define the \newterm{empirical training risk} as
\begin{align}
  \label{eqn:train_erm}
  \hat{R}_{train}(\boldsymbol{f}, \phi) := \frac{1}{NT} \sum_{t=1}^T \sum_{n=1}^N \ell(\pi^{f_t, \phi}(s_{t, n}), a_{t, n}),
\end{align}
and the corresponding minimizer of \eqref{eqn:train_erm} is $\hat{\phi} = \argmin_{\phi \in \Phi} \min_{\boldsymbol{f} \in \gF^{\otimes T}} \hat{R}_{train}(\boldsymbol{f}, \phi)$.

In the transfer phase, we are given $M$ state-action pairs $\{(s_{m}, a_{m})\}_{m=1}^M$ for a target task $\tau$.
With the same loss function $\ell$ as \eqref{eqn:train_erm}, we define the \newterm{empirical test risk} as
\begin{align}
  \label{eqn:test_erm}
  \hat{R}_{test}(f_\tau, \phi) := \frac{1}{M} \sum_{m=1}^M \ell(\pi^{f_\tau, \phi}(s_{m}), a_{m}),
\end{align}
where $f_\tau \in \gF$ is the task-specific mapping for task $\tau$.
We obtain a task-specific mapping that minimizes \eqref{eqn:test_erm} based on the representation $\hat{\phi}$ obtained from the training phase.
That is, $\hat{f}_\tau = \argmin_{f \in \gF} \hat{R}_{test}(f, \hat{\phi})$.

We perform empirical risk minimization (ERM) on both stages---the empirical risks allow us to quantify the generalization error of the learner.
Observing that the log loss corresponds to the KL-divergence, we convert the generalization error to policy error.
Finally, this allows us to establish the sample complexity bound of achieving $\varepsilon$-optimal policy based on the diversity of the source tasks.
The diversity is measured with a positive constant $\sigma$, where larger $\sigma$ corresponds to higher diversity.
We now state our main theorem and defer the analysis to appendix \ref{appendix:theoretical_analysis}:
\begin{theorem}
\label{thm:transfer_il_error_bound}
(Transfer Imitation Learning Policy Error Bound.)
Let $\pi^*_\tau$ be the optimal policy for the target task $\tau$.
Let $\hat{\phi}$ be the ERM of $\hat{R}_{train}$ defined in \eqref{eqn:train_erm} and let $\hat{f}_\tau$ be the ERM of $\hat{R}_{test}$ defined in \eqref{eqn:test_erm} by fixing $\hat{\phi}$.
Let $\sigma > 0$.
Suppose the source tasks are $\sigma$-diverse.
With a deterministic expert policy $\pi^*_\tau$ and under some assumptions, we have that with probability $1 - 2\delta$,
\begin{align}
    \text{Policy Error} = \maxnorm{v^{\pi^*_\tau} - v^{\softmax (\hat{f}_\tau \circ \hat{\phi})}} \leq \frac{2 \sqrt{2}}{(1 - \gamma)^2} \sqrt{\varepsilon_{gen} + 2\zeta},
\end{align}
where $\varepsilon_{gen} = \gO\left(1/\sqrt{\sigma^2 NT}, 1/\sqrt{M}, \fR_{NT}(\Phi) / \sigma\right)$ is the generalization error, $\fR_{NT}(\Phi)$ is the Rademacher complexity, and $\zeta \in (0, 1)$ is a constant related to the policy realizability.
\end{theorem}

Theorem \ref{thm:transfer_il_error_bound} indicates that we can obtain $\varepsilon$-optimal policy through the transfer-learning procedure.
Notably, we can trade-off the number of target data $M$ at the cost of the number of source tasks $T$ and number of training data per task $N$.
In contrast, behavioural cloning can only leverage the target data to train both the representation and the task-specific mapping.
Consequently, if the representation class $\Phi$ is expressive, it is beneficial to pretrain the representation using the existing source data.

\paragraph{Remark 1}
In practice, our policies are neural networks \citep{DBLP:conf/icml/FujimotoHM18,DBLP:conf/icml/HaarnojaZAL18,yarats2020image}---we can consider the last layer as the task-specific mapping $f$ and the remaining layers as the representation $\phi$.
These neural networks include multilayer perceptrons and convolutional neural networks with Lipschitz activation functions (e.g. ReLU, tanh, sigmoid, max-pooling, etc.).
Their Rademacher complexities can be upper bounded based on the amount of training data, number of layers, number of hidden units, and their parameter norms \citep{DBLP:conf/colt/NeyshaburTS15,sokolic2016lessons,DBLP:conf/colt/GolowichRS18,truong2022rademacher}.
Due to lemma 4 of \citet{DBLP:journals/jmlr/BartlettM02}, our result provides a tighter bound than the Gaussian complexity used in existing works by $\gO(\ln NT)$ \citep{DBLP:conf/icml/AroraDKLS20,NEURIPS2020_59587bff,maurer2016benefit}.

%% file: TexFiles/EmpiricalAnalysis.tex
\section{Experiments}
\label{sec:empirical_analysis}
We aim to investigate the following questions regarding MTIL:
\begin{enumerate*}
 \item Can we achieve better imitation performance than behavioural cloning (BC) with less target data?,
 \item How does the number of source tasks or number of source data impact imitation performance?,
 \item How does the number of target data influence on imitation performance compared to the number of source tasks and source data?
\end{enumerate*}

We implement our MTIL procedure with multitask behavioural cloning (MTBC).
MTBC is first pretrained with $N$ source data from $T$ source tasks, then finetuned with $M$ target data (see appendix~\ref{appendix:algorithm_details} for details), whereas BC is trained with only $M$ target data.
The analyses are done on four simulated environments: frozen lake, pendulum, cheetah, and walker.
We convert the continuous action spaces into discrete action spaces described in appendix \ref{appendix:envs}.
The detailed descriptions of the implementation and the varying environmental parameters are found in appendix \ref{sec:implementation_details}.
For the analysis below, we determined that each environment requires $\lvert \gD \rvert$ demonstrations to achieve near-expert performance when training BC from scratch.
Table \ref{tab:data_per_task} specifies $\lvert \gD \rvert$ for each environment.

\begin{table}[t]
	\caption{The number of demonstrations $\lvert \gD \rvert$ for each environment.
  }
   \label{tab:data_per_task}
   \begin{center}
   \begin{tabular}{cccc}
    \multicolumn{1}{c}{Frozen Lake}&\multicolumn{1}{c}{Pendulum}&\multicolumn{1}{c}{Cheetah}  &\multicolumn{1}{c}{Walker}
   \\ \hline
    500 & 1000 & 1000 & 500
   \end{tabular}
   \end{center}
   \vspace{-5mm}
\end{table}

We first investigate questions \emph{(1)} and \emph{(2)}.
We fix the number of target data $M = \lvert \gD \rvert$ while varying the number of source data $N$ and the number of source tasks $T$.
Figure~\ref{fig:main_disc} indicates that as we increase $N$ and $T$, MTBC generally improves and can outperform BC with the same number of target data.
For cheetah and walker, which are more complex environments, MTBC outperforms BC while MTBC and BC have comparable performance in simpler environments.

\begin{figure}[t]
  \centering
  \includegraphics[width=0.95\linewidth]{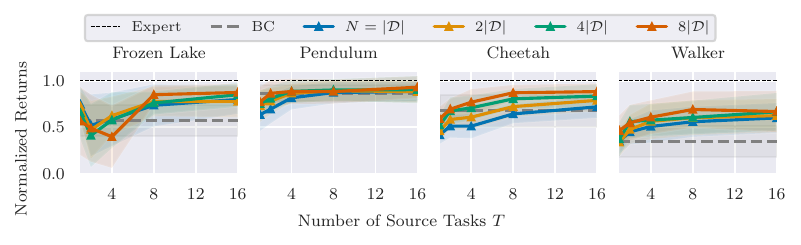}
  \caption{
      \small The performance of MTBC as we vary $N$ and $T$.
      Each solid line colour corresponds to a particular $N$.
      The solid line corresponds to the mean and the shaded region is 1 standard error from the mean.
  }
  \label{fig:main_disc}
\end{figure}

\begin{figure}[t!] %
  \centering
  \includegraphics[width=0.95\linewidth]{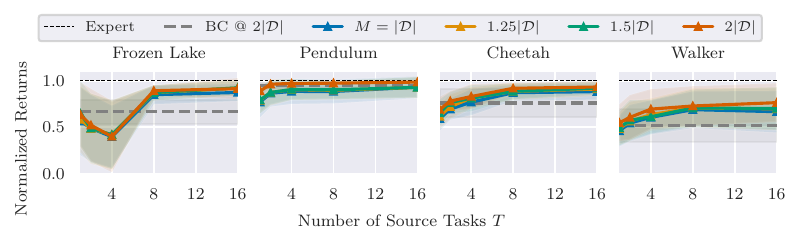}
  \caption{
      \small The performance of MTBC as we vary $M$ and $T$.
      Each solid line colour corresponds to a particular $M$.
      The solid line corresponds to the mean and the shaded region is 1 standard error from the mean.
  }
  \label{fig:vary_target_data_disc}
\end{figure}

To investigate question (3), we fix the amount of source data $N = 8 \lvert \gD \rvert$ while varying the amount of target data $M$ by multiples of $\lvert \gD \rvert$.
For comparison we also include BC trained with $2 \lvert \gD \rvert$ target data.
Figure~\ref{fig:vary_target_data_disc} shows that increasing the number of target data only marginally improves imitation performance.
This result indicates that higher returns are mainly achieved by increasing the amount of source data $N$ and number of source tasks $T$. The latter is due to the learned representation being more expressive than the task-specific mapping.

%% file: TexFiles/Conclusion.tex
\section{Conclusion and Future Directions}
We have theoretically shown a statistical bound to recover the expert policy when transferring representation from a set of source tasks to a target task.
This bound can be directly extended to commonly used neural networks.
We have also conducted experiments to further support our theoretical analysis.
In particular our experiments show that, when we pretrain a representation increasing amount source tasks and source data, one can reduce the amount of target data to imitate the expert policy, compared to training from scratch.
However, we still require a practical method to quantify task diversity prior to training.
Another limitation in our work is assuming that there exists one shared representation.
This assumption can be relaxed by considering the meta-learning formulation where we allow the representation per task to be in a neighbourhood of some reference representation \citep{DBLP:conf/icml/CollinsMOS22}.

%% file: TexFiles/RelatedWork.tex
\section{Related Work}
Our work is heavily inspired by the theoretical analyses on transferring shared representations across tasks \citep{DBLP:conf/icml/AroraDKLS20,NEURIPS2020_59587bff,maurer2016benefit}.
Existing works use Gaussian complexity as the function class complexity measure when analyzing the sample complexity bound.
We use Rademacher complexity instead, which allows us to directly connect our analysis with deep-learning theory \citep{DBLP:journals/actanum/BartlettMR21,truong2022rademacher}.
Another line of works focuses on continuous-control problem \citep{DBLP:journals/csysl/GuoMKP23,DBLP:conf/l4dc/ZhangKLTLTM23}, where they analyze structure-specific systems.
Our work is similar in that we focus on connecting the source tasks and the target task via task diversity, and providing concrete sample complexity bounds based on the task diversity.
However, we remain operating in the infinite-horizon finite MDP setting with non-linear representation class.

%% file: TexFiles/TheoreticalAnalysis.tex
\section{Theoretical Analysis}
\label{appendix:theoretical_analysis}
Our goal is to analyze the sample complexity of obtaining an $\varepsilon$-optimal policy.
The analysis contains the following two steps:
(1) establishing a sample complexity bound for achieving near optimal risk when transferring the representation to a new target task, and
(2) bounding the policy error of the transferred policy through the transfer risk.
The \newterm{transfer risk} is the excess risk on the target task $\tau$:
\begin{align}
  \label{eqn:transfer_risk}
  R_{transfer}(\hat{f}_\tau, \hat{\phi}) = R_{test}(\hat{f}_\tau, \hat{\phi}) - R_{test}(f^*_\tau, \phi^*),
\end{align}
where $R_{phase}(\cdot) = \E\hat{R}_{phase}(\cdot)$, for $phase \in \{train, test\}$ is the expectation of the corresponding empirical risk.
The expectation taken is over the randomness of the state-action pairs.

The first step is heavily inspired by \citet{NEURIPS2020_59587bff} but we only leverage Rademacher complexity to quantify the sample complexity bound.
Our result differs in that we use only Rademacher complexity throughout the analysis due to our assumptions.
The second step leverages the fact that our loss function is the KL-divergence that allows us to upper bound the policy error through the transfer risk.
In general, the results may include few extra constant terms including $C_\Phi, C_\gF$, and $B$ that follow from our assumptions.
We provide the detailed assumptions and proofs in appendix \ref{appendix:proofs}.

First, we compare the training error between the learned representation and true representation.
This error is quantified by their similarity---\citet{NEURIPS2020_59587bff} proposed to consider the task-average prediction difference:
\begin{definition}
    (Task-average Representation Difference \citep{NEURIPS2020_59587bff}.)
    \label{def:task_avg_rep_diff}
    Fix a function class $\gF$, $T$ functions $\boldsymbol{f} = (f_1, \dots, f_T)\in \gF^{\otimes T}$, a loss function $\ell(\cdot, \cdot)$, and data $(s_t, a_t) \sim \mu_{\pi^*_t}$.
    The task-average representation difference between $\phi, \phi' \in \Phi$ is defined as
    \begin{align}
      \label{eqn:task_avg_rep_diff}
      \bar{d}_{\gF}(\phi' ; \phi, \boldsymbol{f}) := \frac{1}{T} \sum_{t=1}^T \inf_{f' \in \gF} \E_{(s_t, a_t)} \left[ \ell((f' \circ \phi')(s_t), a_t) - \ell((f_t \circ \phi)(s_t), a_t) \right].
    \end{align}
  \end{definition}

  The error between the learned representation $\hat{\phi}$ and the true representation $\phi^*$ can be upper bounded by the task-average difference---the bound is dependent on the number of source tasks $T$, the amount of training data per task $N$, and the Rademacher complexity of the representation class.
  \begin{theorem}
    \label{thm:trained_representation_risk_bound}
    (Learned Representation Risk Bound.)
    Let $\hat{\phi}$ be the ERM of $\hat{R}_{train}$ defined in \eqref{eqn:train_erm}.
    Let $\phi^*$ be the true representation and $\boldsymbol{f}^* = (f_1^*, \dots, f_T^*)\in \gF^{\otimes T}$ be the $T$ true task-specific mappings.
    Under some assumptions, we have that with probability at least $1 - \delta$,
    \begin{align*}
      \bar{d}_{\gF}(\hat{\phi}; \phi^*, \boldsymbol{f}^*) \leq 8 \sqrt{2} C_\gF \fR_{NT}(\Phi) + 2B \sqrt{\frac{\log(2/\delta)}{2NT}}.
    \end{align*}
  \end{theorem}

We now consider the test error when using the learned representation on the target task.
This requires a different notion of similarity---\citet{NEURIPS2020_59587bff} proposed to consider the worst-case task-specific mapping in $\gF$.
\begin{definition}
  \label{def:worst_case_rep_diff}
  (Worst-case Representation Difference \citep{NEURIPS2020_59587bff}.)
  Fix a task $\tau$, a function class $\gF$, a loss function $\ell(\cdot, \cdot)$, and data $(s, a) \sim \mu_{\pi^*_\tau}$.
  The worst-case representation difference between $\phi, \phi' \in \Phi$ is defined as
  \begin{align}
    \label{eqn:worst_case_rep_diff}
    d_{\tau,\gF}(\phi' ; \phi) := \sup_{f \in \gF} \inf_{f' \in \gF} \E_{(s, a)} \left[ \ell((f' \circ \phi')(s), a) - \ell((f \circ \phi)(s), a) \right].
  \end{align}
\end{definition}

We can then use definition \ref{def:worst_case_rep_diff} to upper bound the generalization error for the transfer phase ERM estimator:

\begin{theorem}
  \label{thm:transfer_risk_bound}
  (Transfer Risk Bound.)
  Let $\hat{f}_\tau$ be the ERM of $\hat{R}_{test}$ defined in \eqref{eqn:test_erm} with some fixed $\hat{\phi} \in \Phi$.
  Under some assumptions, then with probability $1 - \delta$,
  \begin{align}
    \label{eqn:transfer_risk_bound}
    R_{transfer}(\hat{f}_\tau, \hat{\phi}) \leq 8 C_\gF C_\Phi \sqrt{\frac{\lvert\gA\rvert}{M}} + 2B \sqrt{\frac{\log(2 / \delta)}{2M}} + d_{\tau,\gF}(\hat{\phi}; \phi^*).
  \end{align}
\end{theorem}

We connect the training error of the learned representation $\hat{\phi}$ with the transfer risk via the notion of task diversity similar to the one defined by \citet{NEURIPS2020_59587bff}:
\begin{definition}
  \label{def:diversity}
  ($\sigma$-diversity.)
  Let $\sigma > 0$ and fix a task $\tau$.
  Fix a function class $\gF$, $T$ functions $\boldsymbol{f} = (f_1, \dots, f_T)\in \gF^{\otimes T}$.
  The $T$ tasks are $\sigma$-diverse for representation $\phi$, if for all $\phi' \in \Phi$, we have that $d_{\tau,\gF}(\phi' ; \phi) \leq \bar{d}_\gF(\phi' ; \phi, \boldsymbol{f}) / \sigma$.
\end{definition}
Intuitively, if the $T$ source tasks differ too much from the new task, then inequality only holds with small $\sigma$.
In other words, there is little diversity.
Another perspective is that any $\phi'$ is overfitted to the $T$ tasks and is unable to generalize to the new task.
Conseqeuntly, combining theorems \ref{thm:trained_representation_risk_bound} and \ref{thm:transfer_risk_bound}, we get the following sample complexity bound:
\begin{corollary}
  \label{cor:trained_representation_transfer_risk_bound}
  (Learned Representation Transfer Risk Bound.)
  Let $\hat{\phi}$ be the ERM of $\hat{R}_{train}$ defined in \eqref{eqn:train_erm} and let $\hat{f}_\tau$ be the ERM of $\hat{R}_{test}$ defined in \eqref{eqn:test_erm} by fixing $\hat{\phi}$.
  Suppose the source tasks are $\sigma$-diverse.
  Under some assumptions, we have that with probability $1 - 2\delta$, the transfer risk $R_{transfer}(\hat{f}_\tau, \hat{\phi})$ is upper bounded by:
  \begin{align}
    \label{eqn:representation_transfer_risk_bound}
     \gO \left( C_\gF C_\Phi \sqrt{\frac{\lvert\gA\rvert}{M}} + B \sqrt{\frac{\log(2 / \delta)}{M}} + \frac{1}{\sigma} \left( C_\gF \fR_{NT}(\Phi) + B \sqrt{\frac{\log(2/\delta)}{NT}} \right) \right).
  \end{align}
\end{corollary}

Corollary \ref{cor:trained_representation_transfer_risk_bound} tells us that so long as we can bound the task-average representation difference, we can achieve $\varepsilon$-error with sufficiently large number of demonstration data in the target task.
Furthermore, if the source tasks are not diverse, then we will need more samples from them to reduce the error.
Finally, since the log loss is the KL-divergence with deterministic experts, we can use the result from \citet{NEURIPS2020_b5c01503} to obtain the policy error of the from the transfer-learning procedure:
\begin{theorem}
  \label{thm:policy_error_bound}
  (Policy Error Bound \citep{NEURIPS2020_b5c01503}.)
  Given any two policies $\pi, \pi'$ with $\E_{s \sim \nu_{\pi}}\left[ D_{KL} (\pi(s) \Vert \pi'(s)) \right] < \varepsilon$, we have that $\maxnorm{v^{\pi} - v^{\pi'}} \leq \frac{2 \sqrt{2}}{(1 - \gamma)^2} \sqrt{\varepsilon}$.
\end{theorem}

Thus, by theorem \ref{thm:transfer_il_error_bound}, for specific representation class, we can upper bound their Rademacher complexities and retrieve a concrete sample complexity bound to achieve $\varepsilon$-optimal policies with high probability.

%% file: TexFiles/DetailedProofs.tex
\section{Detailed Proofs}
\label{appendix:proofs}
In this section we provide the proofs for the theorems in appendix \ref{appendix:theoretical_analysis}.
We first list out our assumptions (appendix \ref{appendix:assumptions}), then a few known lemmas and definitions used (appendix \ref{appendix:useful_def_results}), and finally the proofs of our theorems (appendices \ref{appendix:theorem_2}--\ref{appendix:theorem_1}).

\subsection{Assumptions}
\label{appendix:assumptions}
In the analysis, we offload the softmax function to the loss, then we can define a new loss function that can be used for analysis.
That is, let $x \in \sR^D$, then we define the log-softmax loss to be
\begin{align}
  \label{loss:cross_entropy}
  \ell(x, a) = -\log \left(\softmax_a (x)\right),
\end{align}
where $\softmax_a (x)$ corresponds to the $a$'th component of $\softmax (x)$.
We further make the following assumptions for the analysis.
\begin{assumption}
  \label{assum:bounded_representation}
  (Bounded Representation.) The representation $\phi \in \Phi$ is bounded in $\ell_2$-norm: $\Phi \subseteq \{\phi : \gS \to \sR^D \vert \twonorm{\phi(s)} \leq C_\Phi, \forall s \in \gS \}$.
\end{assumption}

\begin{assumption}
  \label{assum:linear_task_mapping}
  (Linear Task-specific Mapping and Bounded Parameters.) The task-specific mapping is linear and is bounded: $\gF = \{f: \sR^D \to \sR^{\lvert \gA \rvert} \vert f = Wx, W \in \sR^{\lvert \gA \rvert \times D}, \frobnorm{W} \leq C_\gF, x \in \sR^D\}$, where $\frobnorm{\cdot}$ is the Frobenius norm
\end{assumption}

\begin{assumption}
  \label{assum:deterministic_expert}
  (Deterministic Expert Policies.)
  The expert policy $\pi^*_\tau$ for each task $\tau$ is deterministic and can be written as $\pi^*_\tau(s) = \boldsymbol{1}_{a = \argmax_{a' \in \gA} (f^*_\tau \circ \phi^*)(s)}$ for all $s \in \gS$.
\end{assumption}

\begin{assumption}
    \label{assum:shared_representation}
    (Shared Representation.) There is a representation $\phi^*$ such that for every task $\tau$, there exists a task-specific mapping $f^*_\tau$ such that the discounted state-action stationary distribution is $\mu_{\pi^*_\tau}(s, a)$.
\end{assumption}

  \begin{assumption}
    \label{assum:realizability}
    (Realizability.) The true shared representation $\phi^*$ is contained in $\Phi$.
    Additionally, for some fixed $\zeta < 1/2$, we have that for all tasks $\tau$, there exists a task-specific mapping $f_\tau \in \gF$ such that for all $s \in \gS$,
    \begin{align*}
      \pi^{f_\tau, \phi^*} \left(  a^* \vert s\right) \geq 1 - \zeta,
    \end{align*}
    where $a^* = \argmax_{a \in \gA} \pi^*_\tau(s)$.
  \end{assumption}
Note that for any infinite-horizon MDPs, there always exists a deterministic optimal policy, thus assumption \ref{assum:deterministic_expert} is often reasonable to obtain.
With the recent successes in foundation models in various applications \citep{bommasani2021opportunities,wei2021finetuned,NEURIPS2022_b1efde53}, both assumptions \ref{assum:shared_representation} and \ref{assum:realizability} may be reasonable.
We note that our results also apply when replacing assumptions \ref{assum:deterministic_expert} and \ref{assum:realizability} with the standard realizability assumption on stochastic expert policies.
Finally, assumptions \ref{assum:bounded_representation} and \ref{assum:linear_task_mapping} are standard regularity conditions in statistical learning theory---we further note that consequently the composed mapping $f \circ \phi$ is bounded: $\sup_{(s, a) \in \gS \times \gA} \lvert (f \circ \phi) (s) \rvert \leq C_\gS$, for any $f \in \gF$ and $\phi \in \Phi$.

\subsection{Useful Definitions and Results}
\label{appendix:useful_def_results}
\begin{definition}
  (Rademacher Complexity.)
  For a vector-valued function $\gF = \{f : \gX \to \sR^K \}$, and $N$ data points $X = (x_1, \dots, x_N)$, where $x_n \in \gX$ for $n \in [N]$, the empirical Rademacher complexity of $\gF$ is defined as
  \begin{align*}
    \hat{\fR}_X(\gF) = \E_\varepsilon\left[ \sup_{f \in \gF} \frac{1}{N} \sum_{n=1}^N \sum_{k=1}^K \varepsilon_{n, k} f_k(x_n) \right],
  \end{align*}
  where $\varepsilon_{n, k}$ are sampled i.i.d. from the Rademacher random variable and $f_k(\cdot)$ is the $k$'th component of $f(\cdot)$.
  Fix a data distribution $\gD_{\gX}$ over $\gX$.
  The corresponding Rademacher complexity of $\gF$ is defined as $\fR_M(\gF) = \E_{\rmX} \left[ \hat{\fR}_\rmX (\gF) \right]$, where the expectation is taken over the distribution $\gD_{\gX}$.
  In the case of $K = 1$, we recover the scalar Rademacher complexity.
\end{definition}
Intuitively, the Rademacher complexity of $\gF$ measures the expressiveness of $\gF$ over all datasets $\rmX$.

\begin{prop}
    \label{prop:log_softmax_boundedness_lipschitzness}
    With assumptions \ref{assum:bounded_representation} and \ref{assum:linear_task_mapping}, the log-softmax loss, defined in \eqref{loss:cross_entropy}, is bounded and Lipschitz (continuous) in its first argument.
  \end{prop}
\begin{proof}
For the first claim, we first fix $(s, a) \in \gS \times \gA$.
Let $\phi_s = \phi(s) \in \sR^D$ and $x_s = W \phi_s$.
Then, we have that
\begin{align*}
    -\log \left(\softmax_a (x_s)\right) = -\left( W \phi_s \right)(a) + \log \sum_{a' \in \gA} \exp \left[ \left( W \phi_s \right)(a') \right].
\end{align*}
Let $a^* = \argmax_{a \in \gA} W\phi_s$ and overload the notation $(W\phi_s)(a) = \phi_s^\top W^\top a$, where $a$ is a one-hot vector with non-zero at the $a$'th entry.
Then, the second term can be upper bounded:
\begin{align*}
    \log \sum_{a' \in \gA} \exp \left[ \phi_s^\top W^\top a' \right] &\leq \log \sum_{a' \in \gA} \exp \left[ \phi_s^\top W^\top a^* \right]\\
    &= \log \lvert \gA \rvert + \phi_s^\top W^\top a^*.
\end{align*}
Thus,
\begin{align*}
    -\log \left(\softmax_a (x_s)\right) &\leq  - \phi_s^\top W^\top a + \log \lvert \gA \rvert +  \phi_s^\top W^\top a^*\\
    &\leq \log \lvert \gA \rvert + 2 \phi_s^\top W^\top a^*.
\end{align*}
Taking the supremum norm over $\gS \times \gA$, we have that
\begin{align*}
    \sup_{(s, a) \in \gS \times \gA} \Big\lvert -\log \left(\softmax_a (x_s)\right) \Big\rvert &\leq \sup_{(s, a) \in \gS \times \gA} \Big\lvert \log \lvert \gA \rvert + 2 \phi_s^\top W^\top a^*_s \Big\rvert\\
    &\leq \log \lvert \gA \rvert + 2 \sup_{(s, a) \in \gS \times \gA} \Big\lvert  \phi_s^\top W^\top a^* \Big\rvert \\
    &\leq \log \lvert \gA \rvert + 2 C_\Phi C_\gF,
\end{align*}
where the last inequality follows upper bounding the second term through H\"older's inequality, setting both norms to be the 2-norm.
This verifies the boundedness of the log-softmax loss.

For the second claim, we can bound the gradient of the log-softmax loss with respect to the first argument and apply mean-value theorem.
Let us first consider the partial derivatives of $\ell(x, a)$.
Let $x_i$ be the $i$'th component of $x$.
For any $x \in \sR^D$, we have that
\begin{align*}
    \frac{\partial}{\partial x_{i}} \ell(x, a) \Big\vert_{i = a} = -\frac{\sum_{a' \neq i} \exp x^\top a'}{\sum_{a' \in \gA} \exp x^\top a'} && \frac{\partial}{\partial x_{i}} \ell(x, a) \Big\vert_{i \neq a} = \frac{\exp x^\top i}{\sum_{a' \in \gA} \exp x^\top a'},
\end{align*}
where we overload the notation and represent $a' \in \gA$ as a one-hot vector.
Then, consider the $\ell_2$-norm of the gradient, we have that:
\begin{align*}
    \twonorm{\nabla_x \ell(x, a)}^2 &= \left(-\frac{\sum_{a' \neq a} \exp x^\top a'}{\sum_{a' \in \gA} \exp x^\top a'}\right)^2 + \frac{\sum_{a' \neq a} \left(\exp x^\top a' \right)^2}{\left(\sum_{a' \in \gA} \exp x^\top a' \right)^2}\\
    &\leq 2 \left(\frac{\sum_{a' \neq a} \exp x^\top a'}{\sum_{a' \in \gA} \exp x^\top a'}\right)^2\\
    &= 2,
\end{align*}
where the first inequality follows from Jensen's inequality.
Consequently, by mean-value theorem, we have that $\lvert \ell(x, a) - \ell(y, a) \rvert \leq \sqrt{2} \twonorm{x - y}$, for any $x, y \in \sR^D$, verifying the Lipschitzness of the log-softmax loss in its first argument.
\end{proof}

\begin{prop}
    \label{prop:log_softmax_linear_lipschitzness}
    Let $\ell$ be the log-softmax loss defined in \eqref{loss:cross_entropy} and $f \in \gF$.
    Under assumptions \ref{assum:bounded_representation} and \ref{assum:linear_task_mapping}, the function $h_a(\phi(s)) = \ell(f(\phi(s)), a)$ is Lipschitz in $\phi(s)$, for any $s \in \gS$.
\end{prop}
\begin{proof}
    We first note that $h_a(\phi_s)$ can be written as $\log \sum_{b \in \gA} \exp \left( b^\top W \phi_s \right) - a^\top W \phi_s$, where $\phi_s = \phi(s) \in \sR^D$, and we overload the notation and write $a, b$ as the one-hot vectors.

    For the first term, the $\ell_2$-norm of its gradient with respect to $\phi_s$ can be upper bounded by $C_\gF$:
    \begin{align*}
        \bigtwonorm{ \nabla_{\phi_s} \log \sum_{b \in \gA} \exp \left( b^\top W \phi_s \right)} &= \bigtwonorm{ \frac{\sum_b \exp \left( b^\top W \phi_s \right) W^\top b}{\sum_b \exp \left( b^\top W \phi_s \right)}}\\
        &= \twonorm{W^\top \softmax(W\phi_s)}\\
        &\leq \frobnorm{W} \twonorm{\softmax(W\phi_s)}\\
        &\leq C_\gF.
    \end{align*}

    For the second term, the $\ell_2$-norm of its gradient with respect to $\phi_s$ can be upper bounded by $C_\gF$:
    \begin{align*}
        \twonorm{ \nabla_{\phi_s} a^\top W \phi_s} &= \twonorm{W^\top a} \leq \frobnorm{W} \twonorm{a} \leq C_\gF.
    \end{align*}

    Since the first term is convex and the second term is linear, $h_a(\phi_s)$ is $2C_\gF$-Lipschitz.
\end{proof}

\begin{definition}
  (Bounded Difference Property.)
  The function $f: \sR^N \to \sR$ satisfies the bounded difference inequality with positive constants $(L_1, \dots, L_N)$ if, for each $n \in [N]$,
  \begin{align}
    \label{eqn:bounded_diff_prop}
    \sup_{x_1, \dots, x_N, x'_n \in \sR} \lvert f(x_1, \dots, x_n, \dots, x_N) - f(x_1, \dots, x'_n, \dots,x_N) \rvert \leq L_n.
  \end{align}
\end{definition}

\begin{lemma}
  (McDiarmid's Inequality/Bounded Difference Inequality.)
  Fix a data distribution $\gD_{\gX}$.
  Suppose $f$ satisfies the bounded difference property defined in \eqref{eqn:bounded_diff_prop}, with positive constants $L_1, \dots, L_N$ and that $X = (X_1, \dots, X_N)$ is drawn independently from $\gD_{\gX}$. Then
  \begin{align*}
    \mathbb{P}\left[ \lvert f(X) - \E f(X) \rvert \geq t \right] \leq 2\exp{\left( \frac{-2t^2}{\sum_{n=1}^N L_n^2} \right)}, \forall t \geq 0.
  \end{align*}
\end{lemma}
\begin{proof}
  We refer the readers to corollary 2.21 of \citet{wainwright2019high} for the proof.
\end{proof}

\begin{theorem}
  \label{thm:rademacher_complexity_bound}
  (Rademacher Complexity Bound.)
  Fix a data distribution $\gD_{\gX}$ and parameter $\delta \in (0, 1)$.
  Suppose $\gF \subseteq \{f: \gX \to [0, B]\}$ and $X = (X_1, \dots, X_N)$ is drawn i.i.d. from $\gD_\gX$.
  Then with probability at least $1 - \delta$ over the draw of $X$, for any function $f \in \gF$,
  \begin{align}
    \Bigg\lvert \E_{x \sim \gD_\gX}\left[ f(x) \right] - \frac{1}{n} \sum_{n=1}^N f(X_n)\Bigg\rvert \leq 2 \fR_N(\gF) + B \sqrt{\frac{\log(2/\delta)}{2n}}.
  \end{align}
\end{theorem}
\begin{proof}
  This result follows closely to the proof of Theorem 10 in~\citet{koltchinskii2002empirical}, where the only modifications come from applying $B$-boundedness of $f$ when applying McDiarmid's inequality (i.e. $L_n \leq \frac{B}{N}, \forall n \in [N]$.)
\end{proof}

\begin{theorem}
  \label{thm:vector_contraction_inequality}
  (Vector-contraction Inequality \citep{DBLP:conf/alt/Maurer16}.)
  Let $x_1, \dots, x_N \in \gX$, $\gF \subseteq \{f: \gX \to \sR^D \}$ be a class of functions, and $h_n : \sR^D \to \sR$ to be $L$-Lipschitz, for all $n \in [N]$.
  Then
  \begin{align*}
    \E \left[ \sup_{f \in \gF} \sum_{n=1}^N \varepsilon_n h_n(f(x_n)) \right] \leq \sqrt{2} L \E \left[ \sup_{f \in \gF} \sum_{d=1}^D \sum_{n=1}^N \varepsilon_{d, n} f_d(x_n) \right],
  \end{align*}
  where $\varepsilon_n, \varepsilon_{d, n}$ are i.i.d. sequences of Rademacher variables, and $f_d(x_n)$ is the $d$'th component of $f(x_n)$.
\end{theorem}
\begin{proof}
  We refer the readers to theorem 3 and corollary 4 of \citet{DBLP:conf/alt/Maurer16} for the proof.
\end{proof}

\subsection{Proof of Theorem \ref{thm:trained_representation_risk_bound}}
\label{appendix:theorem_2}
Theorem \ref{thm:trained_representation_risk_bound} states the following:

\textit{
Let $\hat{\phi}$ be the ERM of $\hat{R}_{train}$ defined in \eqref{eqn:train_erm}.
Let $\phi^*$ be the true representation and $\boldsymbol{f}^* = (f_1^*, \dots, f_T^*)\in \gF^{\otimes T}$ be the $T$ true task-specific mappings.
If the assumptions \ref{assum:bounded_representation} to \ref{assum:realizability} hold, then with probability $1 - \delta$,
\begin{align*}
    \bar{d}_{\gF}(\hat{\phi}; \phi^*, \boldsymbol{f}^*) \leq 8 \sqrt{2} C_\gF \fR_{NT}(\Phi) + 2B \sqrt{\frac{\log(2/\delta)}{2NT}}.
\end{align*}
}

\begin{proof}
    The proof follows closely from the analysis of \citet{NEURIPS2020_59587bff}.
    First, recall that $\hat{\phi}, \hat{\boldsymbol{f}}$ are respectively the ERMs of \eqref{eqn:train_erm} and \eqref{eqn:test_erm}.
    For any $T$ task-specific mappings $\boldsymbol{f}' = (f'_1, \dots, f'_T)$ and a representation $\phi'$.
    Let $\boldsymbol{f}^* = (f_1^*, \dots, f_T^*)$ be the true task-specific mappings and $\phi^*$ be the true shared representation.
    We define the centered training risk and its empirical counterpart respectively as:
    \begin{align*}
        L(\boldsymbol{f}', \phi', \boldsymbol{f}^*, \phi^*) &= \frac{1}{T}\sum_{t=1}^{T} \E_{(s_t, a_t)} \left[ \ell (\pi^{f'_t, \phi'}(s_t), a_t) - \ell (\pi^{f_t^*, \phi^*}(s_t), a_t) \right],\\
        \hat{L}(\boldsymbol{f}', \phi', \boldsymbol{f}^*, \phi^*) &= \frac{1}{T}\sum_{t=1}^{T} \sum_{n=1}^{N} \left( \ell (\pi^{f'_t, \phi'}(s_{t, n}), a_{t, n}) - \E_{(s_t, a_t)} \left[\ell (\pi^{f_t^*, \phi^*}(s_t), a_t) \right] \right).
    \end{align*}
    Define $\tilde{\boldsymbol{f}} = \frac{1}{T}\sum_{t=1}^{T} \arg\inf_{f_t \in \gF} \E_{s_t, a_t} \left[ \ell (\pi^{f_t, \hat{\phi}}(s_t), a_t) - \ell (\pi^{f_t^*, \phi^*}(s_t), a_t) \right]$ to be the minimizer of the centered training risk by fixing $\hat{\phi}$.
    Then, we have that $L(\tilde{\boldsymbol{f}}, \hat{\phi}, \boldsymbol{f}^*, \phi^*) = \bar{d}_\gF(\hat{\phi} ; \phi^*, \boldsymbol{f}^*)$.
    Now, we aim to upper bound $\bar{d}_\gF(\hat{\phi} ; \phi^*, \boldsymbol{f}^*)$ through the difference in the centered training risk:
    \begin{align*}
      \bar{d}_\gF(\hat{\phi} ; \phi^*, \boldsymbol{f}^*) &= 
      L(\tilde{\boldsymbol{f}}, \hat{\phi}, \boldsymbol{f}^*, \phi^*) - L(\boldsymbol{f}^*, \phi^*, \boldsymbol{f}^*, \phi^*)\\
      &= \underbrace{L(\tilde{\boldsymbol{f}}, \hat{\phi}, \boldsymbol{f}^*, \phi^*) - L(\hat{\boldsymbol{f}}, \hat{\phi}, \boldsymbol{f}^*, \phi^*)}_{\leq 0} +
        L(\hat{\boldsymbol{f}}, \hat{\phi}, \boldsymbol{f}^*, \phi^*) - L({\boldsymbol{f}^*}, \phi^*, \boldsymbol{f}^*, \phi^*),
    \end{align*}
    where the first difference is non-positive by definition of $\tilde{\boldsymbol{f}}$.
    Thus, it remains to bound the second difference, which can be done via standard risk decomposition.
    First, recall that:
    \begin{align*}
        \hat{R}_{train}(\boldsymbol{f}, \phi) &= \frac{1}{NT} \sum_{t=1}^T \sum_{n=1}^N \ell(\pi^{f_t, \phi}(s_{t, n}), a_{t, n}),\\
        R_{train}(\boldsymbol{f}, \phi) &= \E\left[ \hat{R}_{train}(\boldsymbol{f}, \phi) \right].
    \end{align*}
    Then, we have that
    \begin{align*}
        L(\hat{\boldsymbol{f}}, \hat{\phi}, \boldsymbol{f}^*, \phi^*) - L(\boldsymbol{f}^*, \phi^*, \boldsymbol{f}^*, \phi^*) =&
        L(\hat{\boldsymbol{f}}, \hat{\phi}, \boldsymbol{f}^*, \phi^*) - \hat{L}(\hat{\boldsymbol{f}}, \hat{\phi}, \boldsymbol{f}^*, \phi^*)\\
        &+
        \hat{L}(\hat{\boldsymbol{f}}, \hat{\phi}, \boldsymbol{f}^*, \phi^*) - \hat{L}(\boldsymbol{f}^*, \phi^*, \boldsymbol{f}^*, \phi^*)\\
        &+
        \hat{L}(\boldsymbol{f}^*, \phi^*, \boldsymbol{f}^*, \phi^*) - L(\boldsymbol{f}^*, \phi^*, \boldsymbol{f}^*, \phi^*)\\
        \numleq{\text{i}}&
        L(\hat{\boldsymbol{f}}, \hat{\phi}, \boldsymbol{f}^*, \phi^*) - \hat{L}(\hat{\boldsymbol{f}}, \hat{\phi}, \boldsymbol{f}^*, \phi^*)\\
        &+
        \hat{L}(\boldsymbol{f}^*, \phi^*, \boldsymbol{f}^*, \phi^*) - L(\boldsymbol{f}^*, \phi^*, \boldsymbol{f}^*, \phi^*)\\
        \leq& 2 \sup_{\boldsymbol{f} \in \gF^{\otimes T}, \phi \in \Phi} \lvert R_{train}(\boldsymbol{f}, \phi) - R_{train}(\boldsymbol{f}^*, \phi^*) \rvert\\
        \numleq{\text{ii}}&
        2 \left( 2\fR_{NT}(\ell \circ \gF \circ \Phi) + B\sqrt{\frac{\log (2/\delta)}{2NT}} \right) \text{ w.p. } 1 - \delta,
    \end{align*}
    where in (i) the second difference is non-positive due to assumptions \ref{assum:deterministic_expert} and \ref{assum:realizability}; (ii) is a result of theorem \ref{thm:rademacher_complexity_bound}, where we used proposition \ref{prop:log_softmax_boundedness_lipschitzness}.

    We now bound the Rademacher complexity term $\fR_{NT}(\ell \circ \gF \circ \Phi)$
    Note that $\ell \circ \gF$ is $2C_\gF$-Lipschitz by proposition \ref{prop:log_softmax_linear_lipschitzness}, thus by theorem \ref{thm:vector_contraction_inequality}, we have that $\fR_{NT}(\ell \circ \gF \circ \Phi) \leq 2\sqrt{2} C_\gF \fR_{NT}(\Phi)$.
    Substituting this upper bound into the above, we get that with probability at least $1 - \delta$,
    \begin{align*}
        \bar{d}_\gF(\hat{\phi} ; \phi^*, \boldsymbol{f}^*) \leq 8\sqrt{2}C_\gF \fR_{NT}(\Phi) + 2B\sqrt{\frac{\log (2/\delta)}{2NT}}.
    \end{align*}
\end{proof}

\subsection{Proof of Theorem \ref{thm:transfer_risk_bound}}
\label{appendix:theorem_3}
Theorem \ref{thm:transfer_risk_bound} states the following:

\textit{
let $\hat{f}_\tau$ be the ERM of $\hat{R}_{test}$ defined in \eqref{eqn:test_erm} with some fixed $\hat{\phi} \in \Phi$.
If the assumptions \ref{assum:bounded_representation} to \ref{assum:realizability} hold, then with probability $1 - \delta$,
\begin{align*}
  R_{transfer}(\hat{f}_\tau, \hat{\phi}) \leq 8 C_\gF C_\Phi \sqrt{\frac{\lvert\gA\rvert}{M}} + 2B \sqrt{\frac{\log(2 / \delta)}{2M}} + d_{\tau,\gF}(\hat{\phi}; \phi^*).
\end{align*}
}
\begin{proof}
  The proof follows closely from the analysis of \citet{NEURIPS2020_59587bff}.
  First recall that:
  \begin{align*}
    R_{transfer}(\hat{f}_\tau, \hat{\phi}) &= R_{test}(\hat{f}_\tau, \hat{\phi}) - R_{test}(f^*_\tau, \phi^*),\\
    R_{test}(f, \phi) &= \E\left[ \hat{R}_{test}(f, \phi) \right],\\
    \hat{R}_{test}(f, \phi) &= \frac{1}{M}\sum_{m=1}^{M} \ell((f \circ \phi)(s_m), a_m),\\
    \hat{f} &= \argmin_{f \in \gF}\hat{R}_{test}(f, \hat{\phi}).
  \end{align*}
  Consider the minimizer of the test risk given the estimated representation $\hat{\phi}$: $\hat{f}^* = \argmin_{f \in \gF} R_{test}(f, \hat{\phi})$.
  Then, we have that
  \begin{align*}
    &R_{test}(\hat{f}_\tau, \hat{\phi}) - R_{test}(f^*_\tau, \phi^*)\\
    =& \underbrace{R_{test}(\hat{f}_\tau, \hat{\phi}) - R_{test}(\hat{f}^*_\tau, \hat{\phi})}_{\text{(a)}} + \underbrace{R_{test}(\hat{f}^*_\tau, \hat{\phi}) - R_{test}(f^*_\tau, \phi^*)}_{\text{(b)}}.
  \end{align*}
  We bound the term (a) via standard risk decomposition:
  \begin{align*}
    &R_{test}(\hat{f}_\tau, \hat{\phi}) - R_{test}(\hat{f}^*_\tau, \hat{\phi})\\
    =& R_{test}(\hat{f}_\tau, \hat{\phi}) - \hat{R}_{test}(\hat{f}_\tau, \hat{\phi})
    + \hat{R}_{test}(\hat{f}_\tau, \hat{\phi}) - \hat{R}_{test}(\hat{f}^*_\tau, \hat{\phi})
    + \hat{R}_{test}(\hat{f}^*_\tau, \hat{\phi}) - R_{test}(\hat{f}^*_\tau, \hat{\phi})\\
    \numleq{\text{i}}& R_{test}(\hat{f}_\tau, \hat{\phi}) - \hat{R}_{test}(\hat{f}_\tau, \hat{\phi})
    + \hat{R}_{test}(\hat{f}^*_\tau, \hat{\phi}) - R_{test}(\hat{f}^*_\tau, \hat{\phi})\\
    \leq& 2 \sup_{f \in \gF} \lvert R_{test}(f, \hat{\phi}) - \hat{R}_{test}(f, \hat{\phi}) \rvert\\
    \numleq{\text{ii}}& 2\left( 2 \fR_M(\ell \circ \gF) + B \sqrt{\frac{\log(2/\delta)}{2M}} \right) \text{ w.p. } 1 - \delta,
  \end{align*}
  where (i) follows from the fact that $\hat{f}$ is the empirical test risk minimizer by definition---clearly $\hat{R}_{test}(\hat{f}_\tau, \hat{\phi}) - \hat{R}_{test}(\hat{f}^*_\tau, \hat{\phi}) \leq 0$;
  (ii) is a result of theorem \ref{thm:rademacher_complexity_bound}.

  We now bound the Rademacher complexity term $\fR_M(\ell \circ \gF)$.
  Note that $\ell$ is $\sqrt{2}$-Lipschitz in its first argument for every $a \in \gA$ from proposition \ref{prop:log_softmax_boundedness_lipschitzness}.
  By theorem \ref{thm:vector_contraction_inequality}, we have that $\fR_M(\ell \circ \gF) \leq 2\fR_M(\gF)$.
  Thus, substituting this upper bound into the above, we get that with probability at least $1 - \delta$,
  \begin{align*}
    R_{test}(\hat{f}_\tau, \hat{\phi}) - R_{test}(\hat{f}^*_\tau, \hat{\phi}) \leq 8 \fR_M(\gF) + 2B \sqrt{\frac{\log(2/\delta)}{2M}}.
  \end{align*}

  Now, by assumption \ref{assum:linear_task_mapping}, $\gF$ is the class of linear functions with bounded Frobenius norm.
  Consider the empirical Rademacher complexity $\hat{\fR}_S(\gF)$:
  \begin{align*}
    \hat{\fR}_S(\gF) &= \frac{1}{M}\E\left[ \sup_{f \in \gF} \sum_{a=1}^{\lvert\gA\rvert} \sum_{m=1}^M \varepsilon_{am} f_a(\phi_m) \right]\\
    &\numleq{\text{i}} \frac{1}{M} C_\gF \sqrt{\lvert\gA\rvert \sum_{m=1}^M \twonorm{\phi_m}^2}\\
    &\numleq{\text{ii}} \frac{1}{M} C_\gF \sqrt{\frac{\lvert\gA\rvert C_\Phi^2}{M}}\\
    &\leq C_\gF C_\Phi \sqrt{\frac{\lvert\gA\rvert}{M}},
  \end{align*}
  where (i) follows from section 4.2 of \citet{DBLP:conf/alt/Maurer16} and applying the Frobenius norm to the inequality; (ii) follows from the fact that $\phi_m = \phi(s_m), s_m \sim \nu_\pi$, thus $\twonorm{\phi_m} \leq C_\Phi$ by assumption \ref{assum:bounded_representation}.
  Consequently, taking the expectation over $X$ we get $\fR_M(\gF) \leq C_\gF C_\Phi \sqrt{\lvert\gA\rvert/M}$.

  We can bound the term (b) as follows:
  \begin{align*}
    & R_{test}(\hat{f}^*_\tau, \hat{\phi}) - R_{test}(f^*_\tau, \phi^*)\\
    =& \inf_{f \in \gF} R_{test}(f, \hat{\phi}) - \inf_{f' \in \gF} R_{test}(f', \phi^*)\\
    \leq & \sup_{f' \in \gF} \inf_{f \in \gF} R_{test}(\hat{f}^*_\tau, \hat{\phi}) - R_{test}(f', \phi^*)\\
    =& d_\gF (\hat{\phi} ; \phi^*).
  \end{align*}

  Finally, by comining both (a) and (b) terms, we conclude that with probability at least $1 - \delta$,
  \begin{align*}
    R_{transfer}(\hat{f}_\tau, \hat{\phi}) \leq 8 C_\gF C_\Phi \sqrt{\frac{\lvert\gA\rvert}{M}} + 2B \sqrt{\frac{\log(2 / \delta)}{2M}} + d_{\tau,\gF}(\hat{\phi}; \phi^*).
  \end{align*}
\end{proof}

\subsection{Proof of Corollary \ref{cor:trained_representation_transfer_risk_bound}}
\label{appendix:corollary_1}
Corollary \ref{cor:trained_representation_transfer_risk_bound} states the following:

\textit{
Let $\hat{\phi}$ be the ERM of $\hat{R}_{train}$ defined in \eqref{eqn:train_erm} and let $\hat{f}_\tau$ be the ERM of $\hat{R}_{test}$ defined in \eqref{eqn:test_erm} by fixing $\hat{\phi}$.
Suppose the source tasks are $\sigma$-diverse.
If the assumptions \ref{assum:bounded_representation} to \ref{assum:realizability} hold, then with probability $1 - 2\delta$, $R_{transfer}(\hat{f}_\tau, \hat{\phi})$ is upper bounded by:
\begin{align*}
    \gO \left( C_\gF C_\Phi \sqrt{\frac{\lvert\gA\rvert}{M}} + B \sqrt{\frac{\log(2 / \delta)}{M}} + \frac{1}{\sigma} \left( C_\gF \fR_{NT}(\Phi) + B \sqrt{\frac{\log(2/\delta)}{NT}} \right) \right).
\end{align*}
}
\begin{proof}
    Set the probability of bad events of theorems \ref{thm:trained_representation_risk_bound} and \ref{thm:transfer_risk_bound} to be $\delta$ each.
    We obtain the desired results by taking complment of the union bound over the bad events and merging the terms.
\end{proof}

\subsection{Proof of Theorem \ref{thm:transfer_il_error_bound}}
\label{appendix:theorem_1}
Theorem \ref{thm:transfer_il_error_bound} states the following:

\textit{
Let $\hat{\phi}$ be the ERM of $\hat{R}_{train}$ defined in \eqref{eqn:train_erm} and let $\hat{f}_\tau$ be the ERM of $\hat{R}_{test}$ defined in \eqref{eqn:test_erm} by fixing $\hat{\phi}$.
Suppose the source tasks are $\sigma$-diverse.
Under assumptions \ref{assum:bounded_representation} to \ref{assum:realizability}, we have that with probability $1 - 2\delta$,
\begin{align}
  \maxnorm{v^{\pi^*_\tau} - v^{\softmax (\hat{f}_\tau \circ \hat{\phi})}} \leq \frac{2 \sqrt{2}}{(1 - \gamma)^2} \sqrt{\varepsilon_{gen} + 2\zeta},
\end{align}
where $\varepsilon_{gen}$ is the RHS of \eqref{eqn:representation_transfer_risk_bound}.
}

\begin{proof}
  Our goal is to apply theorem \ref{thm:policy_error_bound}.
  Let $\pi = \pi^*$ and $\pi' = \softmax (\hat{f}_\tau \circ \hat{\phi})$, we need to bound their expected Kullback–Leibler (KL) divergence.
  To accomplish this, note that theorem \ref{thm:transfer_risk_bound} essentially means that we can bound the test risk:
  \begin{align}
    R_{test}(\hat{f}_\tau, \hat{\phi}) \leq \varepsilon_{gen} + \varepsilon_{best},
  \end{align}
  where $\varepsilon_{gen}$ is the RHS of \eqref{eqn:representation_transfer_risk_bound} and $\varepsilon_{best} = R_{test}(f^*_\tau, \phi^*)$.
  Indeed, $R_{test}(\hat{f}_\tau, \hat{\phi})$ is the expected KL divergence between $\pi^*_\tau(s)$ and $\softmax (\hat{f}_\tau \circ \hat{\phi})(s)$:
  \begin{align*}
    R_{test}(\hat{f}_\tau, \hat{\phi}) &= \E_{(s, a) \sim \mu_{\pi^*_\tau}} \left[ -\log \softmax_a (\hat{f}_\tau \circ \hat{\phi})(s) \right]\\
    &= \E_{s \sim \nu_{\pi^*}} \left[ D_{KL}(\pi^*_\tau(s) \Vert \softmax (\hat{f}_\tau \circ \hat{\phi})(s)) \right].
  \end{align*}
  Finally, note that assumptions \ref{assum:deterministic_expert} and \ref{assum:realizability} imply that we can assume $f^* \in \gF$, thus we have that:
  \begin{align*}
    \varepsilon_{best} &= R_{test}(f^*, \phi^*)\\
    &\leq \min_{f \in \gF} R_{test}(f, \phi^*)\\
    &\leq \E_{(s, a) \sim \mu_{\pi^*_\tau}} \left[ -\log (1 - \zeta) \right]\\
    &\leq 2\zeta,
  \end{align*}
  where the third lines come from assumption \ref{assum:realizability} and last inequality comes from $-\log(1 - x) \leq 2x$ for $x \leq 1/2$.
  Finally, we upper bound $\varepsilon_{gen}$ using corollary \ref{cor:trained_representation_transfer_risk_bound} and get the desired result.
\end{proof}

%% file: TexFiles/AlgorithmDetails.tex
\section{Algorithm Details}
\label{appendix:algorithm_details}
{In this section we provide details on the implementation of the multitask behaviour cloning (MTBC) algorithm used to obtain the results highlighted in section~\ref{sec:empirical_analysis}.
Recall that in the training phase we aim to obtain a shared representation $\hat{\phi}$ by minimizing $\hat{R}_{train}$, as described in \eqref{eqn:train_erm}.
As proposed by \citep{DBLP:conf/icml/AroraDKLS20}, this can be done by solving the following bi-level optimization objective:
\begin{align}
  \label{eqn:bilevel_obj}
  \hat{\phi} = \argmin_{\phi} \frac{1}{T} \sum_{t=1}^T \min_{\pi \in \Pi^\phi} \frac{1}{N} \sum_{n = 1}^N \ell(\pi(s_{t, n}), a_{t, n}),
\end{align}
where $(s_{t, n}, a_{t, n})$ is the $n$'th state-action pair in the $t$'th expert dataset and $\ell : \Delta^\gA \times \gA \to [0, \infty)$ is the loss function.
MTBC optimizes \eqref{eqn:bilevel_obj} through a gradient-based approach on a joint objective \eqref{eqn:practical_bilevel_obj} (see algorithm \ref{alg:mtbc_pretrain}).
During the transfer phase, MTBC fixes the representation $\hat{\phi}$ and minimizes $\hat{R}_{test}(f, \hat{\phi})$, as described in \eqref{eqn:test_erm}---this can be done via BC where we set the loss function $\ell$ to be the cross-entropy loss for discrete action space and mean-squared loss for continuous action space.
In practice, since $f$ is a linear function parameterized by $W$, MTBC propagates the gradient to update $W$ (see algorithm \ref{alg:mtbc_transfer}).

\begin{algorithm}[t]
  \caption{Multitask Behavioural Cloning (MTBC): Training Phase}
  \label{alg:mtbc_pretrain}
  \begin{algorithmic}[1]
  \State \textbf{input}: Number of epochs $K$, $T$ expert datasets ${\{\gD_t\}}_{t = 1}^{T}$, where $\lvert \gD_t \rvert = N$, and learning rates $\eta_\phi$ and $\eta_W$.
  \State Initialize parameters for each of the $T$ task-specific mappings: ${\{W_{t}^{(0)}\}}_{t=1}^T$.
  \State Initialize shared representation parameters $\phi^{(0)}$.
  \For{$k = 1, \dots, K$}
  \State Compute loss:
  \begin{align}
    \label{eqn:practical_bilevel_obj}
    \gL \left( \phi^{(k)}, {\{W^{(k)}_t\}}_{t=1}^T \right) = \frac{1}{NT} \sum_{t = 1}^T \sum_{n=1}^N \ell \left( \pi^{W_t^{(k - 1)}, \phi^{(k - 1)}}(s_{t, n}), a_{t, n} \right).
  \end{align}
  \State Update parameters:
  \begin{align*}
    \phi^{(k)} &= \phi^{(k - 1)} - \eta_\phi \nabla_\phi \gL \left( \phi^{(k)}, \{W^{(k)}_t\}_{t=1}^T \right),\\
    W_t^{(k)} &= W_t^{(k - 1)} - \eta_W \nabla_{W_t} \gL \left(\phi^{(k)}, \{W^{(k)}_t\}_{t=1}^T \right), \forall t \in [T].
  \end{align*}
  \EndFor
  \State \textbf{return} $\phi$.
  \end{algorithmic}
\end{algorithm}

\begin{algorithm}[t]
  \caption{Multitask Behavioural Cloning (MTBC): Transfer Phase}\label{alg:mtbc_transfer}
  \begin{algorithmic}[1]
  \State \textbf{input}: Representation $\phi$, number of epochs $K$, expert dataset $\gD$, where $\lvert \gD \rvert = M$, and learning rate $\eta_W$.
  \State Initialize parameters of the task-specific mapping: $W_\tau^{(0)}$.
  \For{$k = 1, \dots, K$}
  \State Compute loss: $\gL \left( \phi, W_\tau^{(k)} \right) = \frac{1}{M} \sum_{m=1}^M \ell \left( \pi^{W_\tau^{(k - 1)}, \phi}(s_{m}), a_{m} \right)$
  \State Update parameters: $W_\tau^{(k)} = W_\tau^{(k - 1)} - \eta_W \nabla_{W} \gL \left(\phi, W_\tau^{(k)} \right)$.
  \EndFor
  \State \textbf{return} $W^{(K)}$.
  \end{algorithmic}
\end{algorithm}

%% file: TexFiles/ImplementationDetails.tex
\section{Implementation Details}
\label{sec:implementation_details}
In this section, we describe in detail the implementation of our experimental analysis.

\subsection{Environments}
\label{appendix:envs}

\begin{figure}[t]
   \centering
   \includegraphics[width=0.19\linewidth]{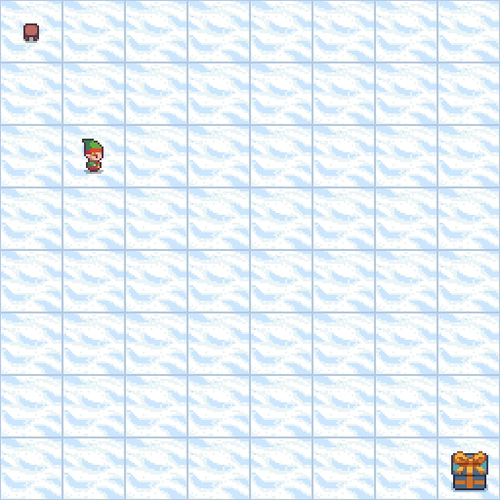}
   \includegraphics[width=0.19\linewidth]{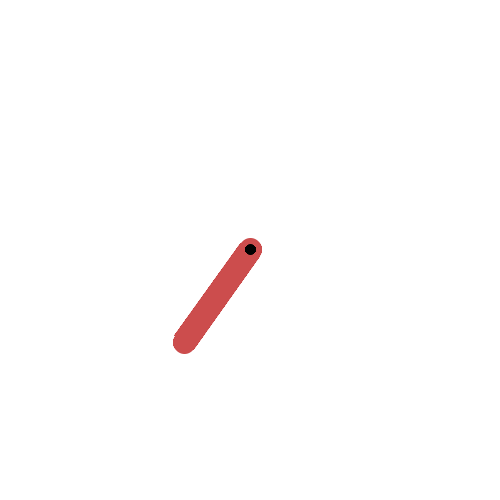}
   \includegraphics[width=0.19\linewidth]{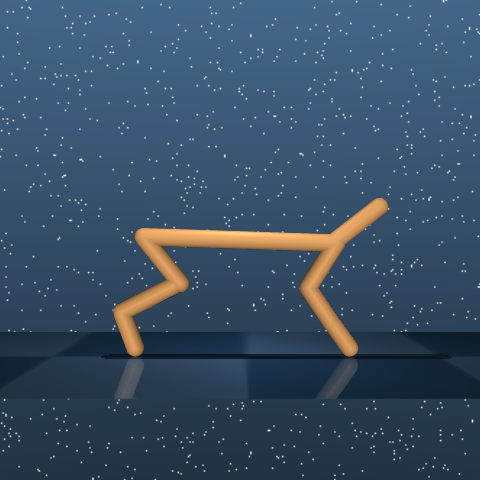}
   \includegraphics[width=0.19\linewidth]{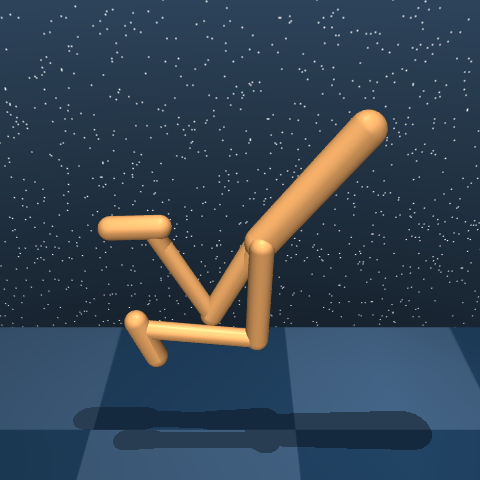}
   \caption{
      The set of environments used for the experiments.
      From left to right: frozen lake, pendulum, cheetah, and walker.
   }
   \label{fig:environments}
\end{figure}

We perform our experiments on four tasks: frozen lake, pendulum, cheetah run, and walker walk.
The first two environments are based on Gymnasium \citep{towers_gymnasium_2023} and the latter two are based on Deepmind control suite \citep{tunyasuvunakool2020}.
To bridge the theory and practice, we provide a discrete action space variant for the continuous environments to validate our hypotheses.
While in general discretizing action space may exclude the true optimal policy \citep{DBLP:conf/icml/DadashiHVGRGP22,NEURIPS2021_e46be61f,seyde2022solving}, our goal is to demonstrate that MTBC can obtain the expert policy with less target data when compared with BC.
Thus, we argue that not having the true optimal policy still validates our goal so long as the expert policy is non-trivial.
To generate multiple source tasks with shared state and action spaces, we modify phyiscal properties of the environment.
All environmental parameters are sampled from uniform distributions of bounded ranges.

\paragraph{Frozen Lake}
The frozen lake task requires an agent the navigate through a $8 \times 8$ grid with non-deterministic transitions.
The action space is the set $\{\text{LEFT}, \text{DOWN}, \text{RIGHT}, \text{UP}, \text{STAY}\}$.
To vary the tasks, we modify the initial state, the goal state, and the transition function.

\paragraph{Pendulum}
The pendulum task requires an agent to swing and keep the link upright.
The default action space is the torque of the revolute joint between $[-2, 2]$.
The discretized action space is the set $\{0, \pm 2^{-3}, \dots, , \pm 2^{1} \}$ (i.e. $11$ actions.)
This allows the agent to perform large-magnitude action for swinging the link up and to perform low-magnitude action for keeping the link upright.
To vary the tasks, we modify the maximum torque applied to the joint.

\paragraph{Locomotion}
Both the cheetah-run and walker-walk tasks require an agent to move above a specified velocity.
The latter task further scales the velocity based on the height of the agent.
The default action space is the torque applied to each of the joint, all bounded between $[-1, 1]$.
The discretized action space is the set $\{-1, 1\}$ per joint (i.e. Bang-Bang control \citep{NEURIPS2021_e46be61f},) thus we have $2^{\dim (\gA)}$ actions after discretizing the action space---we emphasize that we use softmax policies as opposed to per-dimension Bernoulli policies.
To vary the source tasks, we modify the links' size and the joint parameters.

\subsection{Generating the Experts and Demonstrations}
In this section we describe how we generate the experts for all environments.
We first obtain the expert policy $\pi^{default}$ for each environment using the default environmental parameters.
Then, to obtain the expert policy $\pi^{new}$ for each environmental variant, we initialize $\pi^{new}$ using $\pi^{default}$ and resume training.
This pretraining strategy speeds up training on new envrionment variant, as opposed to training a new policy from scratch.
We use Proximal Policy Optimization (PPO) \citep{schulman2017proximal} to train the expert policies for all environments.
In general, all PPO policies are trained using Adam optimizer \citep{DBLP:journals/corr/KingmaB14}.
We further use observation normalization and advantage normalization, as commonly done in practice \citep{engstrom2020implementation,hsu2020revisiting}.

To gather the demonstrations, we execute the expert policies and deterministically select the action with highest logits.
In practice, there is a time-out $H$ for all environments even though in theory they are infinite-horizon MDPs.
We gather a total of $\lceil N/H \rceil$ length $H$ time-out episodes and trim the extra $N \bmod H$ samples, which is more practical and time efficient, and similar to how practitioners leverage data in real-life applications.